\documentclass{article} 
\usepackage{iclr2015,times}
\usepackage{hyperref}
\usepackage{url}

\usepackage{graphicx}
\usepackage{natbib}
\usepackage{amsmath}
\usepackage{amsthm}
\usepackage{amssymb}

\newcommand{\CC}{\mathcal{C}}
\newcommand{\LL}{\mathcal{L}}

\newcommand{\bmx}[0]{\begin{bmatrix}}
\newcommand{\emx}[0]{\end{bmatrix}}

\newcommand{\vect}[1]{\mathbf{#1}}

\newcommand{\matr}[1]{\mathbf{#1}}

\newcommand{\vo}[0]{\vect{o}}
\newcommand{\vb}[0]{\vect{b}}
\newcommand{\vc}[0]{\vect{c}}
\newcommand{\vh}[0]{\vect{h}}

\newcommand{\vx}[0]{\vect{x}}

\newcommand{\vy}[0]{\vect{y}}

\newcommand{\mW}[0]{\matr{W}}

\newcommand{\mV}[0]{\matr{V}}

\newcommand{\sigmoid}{\sigma}
\newcommand{\EE}[0]{\mathbb{E}}

\newtheorem{theorem}{Theorem}[]

\title{Techniques for Learning Binary \\ Stochastic Feedforward Neural Networks}

\author{
Tapani Raiko \& Mathias Berglund \\
Department of Information and Computer Science\\
Aalto University\\
Espoo, Finland \\
\texttt{\{tapani.raiko,mathias.berglund\}@aalto.fi} \\
\And
Guillaume Alain \& Laurent Dinh \\
Department of Computer Science and Operations Research \\
Universit\'{e} de Montr\'{e}al \\
Montr\'{e}al, Canada \\
\texttt{guillaume.alain.umontreal@gmail.com, dinhlaur@iro.umontreal.ca} \\
}

\iclrfinalcopy 

\iclrconference 

\begin{document}

\maketitle

\begin{abstract}
Stochastic binary hidden units in a multi-layer perceptron (MLP) network give at least three potential benefits when compared to deterministic MLP networks. (1) They allow to learn one-to-many type of mappings. (2) They can be used in structured prediction problems, where modeling the internal structure of the output is important. (3) Stochasticity has been shown to be an excellent regularizer, which makes generalization performance potentially better in general. However, training stochastic networks is considerably more difficult.
We study training using $M$ samples of hidden activations per input. We show that the case $M=1$ leads to a fundamentally different behavior where the network tries to avoid stochasticity. We propose two new estimators for the training gradient and propose benchmark tests for comparing training algorithms. Our experiments confirm that training stochastic networks is difficult and show that the proposed two estimators perform favorably among all the five known estimators.
\end{abstract}

\section{Introduction}

Feedforward neural networks, or multi-layer perceptron (MLP) networks, model mappings from inputs $\vx$ to outputs $\vy$ through hidden units $\vh$. 
Typically the network output defines a simple (unimodal) distribution such as an isotropic Gaussian or a fully factorial Bernoulli distribution. 
In case the hidden units are deterministic (using a function $\vx\rightarrow \vh$ as opposed to a distribution $P(\vh|\vx)$), the conditionals $P(\vy|\vx)$ belong to the same family of simple distributions.

Stochastic feedforward neural networks (SFNN) \citep{neal1990learning,neal1992connectionist}
have the advantage when the conditionals $P(\vy|\vx)$ are more complicated. 
While each configuration of hidden units $\vh$ produces a simple output, the mixture over them can approximate any distribution, including multimodal distributions required for one-to-many type of mappings.
In the extreme case of using empty vectors as the input $\vx$, they can be used for unsupervised learning of the outputs $\vy$.

Another potential advantage of stochastic networks is in generalization performance.
Adding noise or stochasticity to the inputs of a deterministic neural network has been found useful as a regularization method \citep{sietsma1991creating}.
Introducing multiplicative binary noise to the hidden units \citep[dropout, ][]{hinton2012improving} regularizes even better.

Binary units have additional advantages in certain settings. For instance conditional computations require hard decisions \citep{bengio2013estimating}. In addition, some harwdare solutions are restricted to binary outputs  \citep[e.g. the IBM SyNAPSE, ][] {esser2013cognitive}.

The early work on SFNNs approached the inference of $\vh$ using Gibbs sampling \citep{neal1990learning,neal1992connectionist} or mean field
\citep{saul1996mean}, which both have their downsides. 
Gibbs sampling can mix poorly and the mean-field approximation can both be inefficient and optimize a lower bound on the likelihood that may be too loose.
More recent work proposes simply drawing samples from $P(\vh|\vx)$ during the feedforward phase \citep{hinton2012improving,bengio2013estimating,tang2013learning}.
This guarantees independent samples and an unbiased estimate of $P(\vy|\vx)$.

We can use standard back-propagation when using stochastic continuous-valued units (e.g.\ with additive noise or dropout),
but back-propagation is no longer possible with discrete units.
There are several ways of estimating the gradient in that case.
\citet{bengio2013estimating} proposes two such estimators: an unbiased estimator with a large variance, and a biased version that approximates back-propagation.

\citet{tang2013learning} propose an unbiased estimator of a lower bound that works reasonably well in a hybrid network containing both deterministic and stochastic units. 
Their approach relies on using more than one sample from $P(\vy|\vx)$ for each training example,
and in this paper we provide theory to show that using more than one sample is an important requirement.
They also demonstrate interesting applications such as mapping the face of a person into varying expressions,
or mapping a silhouette of an object into a color image of the object.

\citet{tang2013learning} argue for the choice of a hybrid network structure based on the finite (and thus limited) number of hidden configurations in a fully discrete $\vh$. 
However, we offer an alternate hypothesis: It is much easier to learn a deterministic network around a small number of stochastic units,
so that it might not even be important to train the stochastic units properly. In an extreme case, the stochastic units are not trained at all, and the deterministic units do all the work.

In this work, we take a step back and study more rigorously the training problem with fully stochastic networks.
We compare different methods for estimating the gradient and propose two new estimators. One is an approximate back-propagation with less bias than the one by \citet{bengio2013estimating},
and the other is a modification of the estimator by \citet{tang2013learning} with less variance.
We propose a benchmark test setting based on the well-known MNIST data and the Toronto Face Database.

\section{Stochastic feedforward neural networks}

We study a model that maps inputs $\vx$ to outputs $\vy$ through stochastic binary hidden units $\vh$. The equations are given for just one hidden layer, but the extension to multiple layers is easy\footnote{Apply $P(\vh|\vx)$ separately for each layer such that $\vx$ denotes the layer below instead of the original input.}. The activation probability is computed just like the activation function in deterministic multilayer perceptron (MLP) networks:
\begin{align}
 P(h_i=1 \mid \vx) = \sigmoid(a_i) = \sigmoid(\mW_{i:} \vx + b_i),
 \label{eq:h_given_x}
\end{align}
where $\mW_{i:}$ denotes the $i$th row vector of matrix $\mW$ and $\sigmoid(\cdot)$ is the sigmoid function.
For classification problems, we use softmax for the output probability
\begin{align}
 P(y=i \mid \vh) = \frac{\exp(\mV_{i:} \vh + c_i)}{\sum_j \exp(\mV_{j:} \vh + c_j)}.
\end{align}
For predicting binary vectors $\vy$, we use a product of Bernoulli distributions
\begin{align}
 P(y_i=1 \mid \vh) = \sigmoid(\mV_{i:} \vh + c_i).
\end{align}

The probabilistic training criterion for deterministic MLP networks is $\log P(\vy| \vx)$. Its gradient with respect to model parameters ${\boldsymbol{\theta}}=\{\mW,\mV,\vb,\vc\}$ can be computed using the back-propagation algorithm, which is based on the chain rule of derivatives. Stochasticity brings difficulties in both estimating the training criterion and in estimating the gradient.
The training criterion of the stochastic network
\begin{align}
  \label{eq:criterion}
  \CC &= \log P(\vy\mid \vx) = \log \sum_\vh P(\vy,\vh \mid \vx) 
    = \log \sum_\vh P(\vy\mid \vh) P(\vh \mid \vx) \\
    &= \log \EE_{P(\vh \mid \vx)} P(\vy\mid \vh)
\end{align}  
requires summation over an exponential number of configurations of $\vh$. Also,
derivatives with respect to discrete variables cannot be directly defined. We will review and propose solutions to both problems below.

\subsection{Proposed estimator of the training criterion}


We propose to estimate the training criterion in Equation \eqref{eq:criterion} by
\begin{align}
  \label{eq:criterion_estimate}
  \hat{\CC}_M &= \log \frac1M \sum_{m=1}^M P(\vy\mid \vh^{(m)}) \\
  \vh^{(m)} &\sim P(\vh \mid \vx).
\end{align}
This can be interpreted as the performance of a finite mixture model over $M$ samples drawn from $P(\vh|\vx)$.

One could hope that using just $M=1$ sample just like in many other stochastic networks \citep[e.g.][]{hinton2012improving} would work well enough. However, here we show in that case the network always prefers to minimize the stochasticity, for instance by increasing the input weights to a stochastic sigmoid unit such that it behaves as a deterministic step-function nonlinearity.

\begin{theorem}
When maximizing the expectation of $\hat{\CC}_1$ in Equation (\ref{eq:criterion_estimate}) using $M=1$, a hidden unit $h_i$ never prefers a stochastic output over a deterministic one. However, when maximizing the expectation of $\CC$ in Equation (\ref{eq:criterion}), the hidden unit $h_i$ may prefer a stochastic output over any of the deterministic ones.
\end{theorem}

\begin{proof}
The expected $\hat{\CC}_1$ over the data distribution can be upper-bounded as
\begin{align}
  \EE_{P_d(\vx,\vy)} \EE_{P(\vh\mid\vx)} \left[ \hat{\CC}_1 \right]
    &=  \EE_{P_d(\vx,\vy)} \EE_{P(\vh\mid\vx)} \log P(\vy\mid\vh) \\
    &= \EE_{P_d(\vx)} \EE_{P(h_i\mid \vx)} \left[ \EE_{P_d(\vy\mid \vx)} \EE_{P(\vh_{\backslash i}\mid \vx,h_i)} \log P(\vy \mid \vh)\right] \\
    &= \EE_{P_d(\vx)} \EE_{P(h_i\mid \vx)} f(h_i,\vx) \\
    &\leq  \EE_{P_d(\vx)} \max_{h_i} f(h_i,\vx),
\end{align}
where $P_d$ denotes the data distribution.
The value in the last inequality is achievable by selecting the distribution of $P(h_i | \vx)$
to be a Dirac delta around the value $h_i$ which maximizes the deterministic function $f(h_i, x)$.
This can be done for every $x$ under the expectation in an independent way.
This is analogous to a idea from game theory:
since the performance achieved with $P(h_i| \vx)$ is a linear combination of
the performances $f(h_i,\vx)$ of the deterministic choices $h_i$,
any mixed strategy $P(h_i| \vx)$ cannot be better than the best deterministic choice.

Let us now look at the situation for the expectation $\CC$
and see how it differs from the case of $\CC_1$ that we had with one particle.
We can see that the original training criterion
can be written as the expectation of a KL-divergence.
\begin{align}
\EE_{P_d(\vx,\vy)} \left[ \CC \right] & = \EE_{P_d(\vx)} \EE_{P_d(\vy \mid \vx)} \log P(\vy\mid\vx) \label{eqn:original_training_criterion} \\
                                  & = \EE_{P_d(\vx)} \left[ -\textrm{KL} \left( P_d(\vy \mid \vx) \| P(\vy\mid\vx) \right)  + \textrm{const} \right]
\end{align}
The fact that this expression features a negative KL-divergence means that the maximum is achieved when the conditionals match exactly.
That is, it it maximized when we have that $P(\vy | \vx) = P_d(\vy | \vx)$ for each value of $\vx$.

We give a simple example in which $(x, h, y)$ each take
values in $\{0, 1\}$. We define the following conditions on
$P_d(y|x)$ and $P(y|h)$, and we show how any deterministic
$P(h|x)$ is doing a bad job at maximizing (\ref{eqn:original_training_criterion}).
\begin{equation*}
    P_d(y\mid x) = \left[\begin{array}{cc}
        0.5 & 0.5\\
        0.5 & 0.5
    \end{array}\right], \hspace{1em}
    P(y\mid h) = \left[\begin{array}{cc}
        0.9 & 0.1\\
        0.1 & 0.9
    \end{array}\right]
\end{equation*}
A deterministic $P(h\mid x) = \left[\begin{array}{cc}
        a & b\\
        1-a & 1-b
    \end{array}\right]$ is one in which $a,b$
take values in $\{0,1\}$.

Criterion (\ref{eqn:original_training_criterion}) is maximized by $(a,b)=(0.5,0.5)$,
regardless of the distribution $P_d(x)$. For the purposes
of comparing solutions, we can simply take $P_d(x)=\left[ 0.5 \hspace{0.4em} 0.5 \right]$.
In that case, we get that the expected $\CC$
takes the value $\hspace{0.3em} 0.5 \log(0.5) + 0.5 \log(0.5) \approx -0.30$.
On the other hand, 
all the deterministic solutions yield a lower value
$\hspace{0.3em} 0.5 \log(0.9) + 0.5 \log(0.1) \approx -0.52$.

\end{proof}

\subsection{Gradient for training $P(\vy|\vh)$}
\label{sec:P_y_given_h}

We will be exploring five different estimators for the gradient of a training criterion wrt.\ parameters ${\boldsymbol{\theta}}$. 
However, all of them will share the following gradient for training $P(\vy|\vh)$.

Let $\vo = \mV \vh + \vc$ be the incoming signal to the activation function $\phi(\cdot)$ in the final output layer. For training $P(\vy|\vh)$, we compute the gradient of the training criterion $\hat{\CC}_M$ in Equation \eqref{eq:criterion_estimate}
\begin{align}
  \hat{\CC}_M &= \log \frac1M \sum_{m=1}^M P(\vy\mid \vh^{(m)}) 
      = \log \frac1M \sum_{m=1}^M \phi(\vo^{(m)}) \\
  G(\vo^{(m)}) & := \frac{\partial \hat{\CC}_M}{\partial \vo^{(m)}} = \frac{\phi^\prime(\vo^{(m)})} {\sum_{m^\prime=1}^M \phi(\vo^{(m^\prime)})}
 = \frac{\phi(\vo^{(m)})}{\sum_{m^\prime=1}^M \phi(\vo^{(m^\prime)})} \frac{\partial \log \phi(\vo^{(m)})}{\partial \vo^{(m)}} \\
  &= \frac{P(\vy\mid \vh^{(m)})}{\sum_{m\prime=1}^M P(\vy\mid \vh^{(m^\prime)})} \frac{\partial \log \phi(\vo^{(m)})}{\partial \vo^{(m)}}
  = \frac{w^{(m)}}{\sum_{m^\prime=1}^M w^{(m^\prime)}
}\frac{\partial \log \phi(\vo^{(m)})}{\partial \vo^{(m)}},
\end{align}
where $w^{(m)}=P(\vy|\vh^{(m)})$ are unnormalized weights. 
In other words, we get the gradient in the mixture by computing the gradient of the individual contribution $m$ and multiplying it with normalized weights $\bar{w}^{(m)}= w^{(m)}/\sum_{m^\prime=1}^M w^{(m^\prime)}$. The normalized weights $\bar{w}^{(m)}$ can be interpreted as responsibilities in a mixture model \citep[see e.g.][Section 2.3.9]{bishop2006pattern}.

The gradients $G(\mV)$, $G(\vh)$, and $G(\vc)$ are computed from $G(\vo)$ using the chain rule of derivatives just like in standard back-propagation.

\subsection{First estimators of the gradient for training $P(\vh|\vx)$}

\citet{bengio2013estimating} proposed two estimators of the gradient for $P(\vh|\vx)$. The first one is unbiased but has high variance. It is defined as
\begin{align}
  G_1(a_i) &:= (h_i-\sigma(a_i))(L-\bar{L}_i) \\
  \bar{L}_i &= \frac{\EE \left[ (h_i-\sigma(a_i))^2 L\right]}{\EE \left[ (h_i-\sigma(a_i))^2 \right]},
\end{align}
where we plug in $L=\hat{\CC}_M$ as the training criterion. We estimate the numerator and the denominator of $\bar{L}_i$ with an exponential moving average.

The second estimator is biased but has lower variance. It is based on back-propagation where we set $\frac{\partial h_i}{\partial a_i} := 1$ resulting in
\begin{align}
 G_2(a_i) := G(h_i) = (\mV_{:i})^T G(\vo).
\end{align}

\subsection{Proposed biased estimator of the gradient for training $P(\vh|\vx)$}
\label{sec:proposed_biased}

We propose a new way of propagating the gradient of the training criterion through discrete hidden units $h_i$. Let us consider $h_i$ continuous random variables with additive noise $\epsilon_i$
\begin{align}
  h_i & = \sigmoid(a_i) + \epsilon_i \\
  \label{eq:epsilon}
  \epsilon_{i} & \sim\begin{cases}
                      1-\sigma(a_{i}) & \hspace{0.5em}\textrm{with probability}\hspace{1em}\sigma(a_{i})\\
                       -\sigma(a_{i}) & \hspace{0.5em}\textrm{with probability}\hspace{1em}1-\sigma(a_{i})
                      \end{cases}
\end{align}
Note that $h_i$ has the same distribution as in Equation (\ref{eq:h_given_x}), that is, it only gets values 0 and 1.
With this formulation, we propose to back-propagate derivatives through $h_i$ by
\begin{align}
 G_3(a_i) := \sigmoid^\prime(a_i) G(h_i) = \sigmoid^\prime(a_i) (\mV_{:i})^T G(\vo)
\end{align}
This gives us a biased estimate of the gradient since we ignore
the fact that the structure of the noise $\epsilon_i$ depends on the input signal $a_i$. 
One should note, however, that the noise is zero-mean with any input $a_i$,
which should help keep the bias relatively small.

\subsection{Variational training}

\citet{tang2013learning} use a variational lower bound $\LL(Q)$ on the training criterion $\CC$ as
\begin{align}
  \CC &= \log P(\vy\mid \vx)
  = \sum_\vh P(\vh\mid \vy,\vx) \log \frac{P(\vy,\vh\mid \vx)}{P(\vh\mid\vy,\vx)}
  \geq \sum_\vh Q(\vh) \log \frac{P(\vy,\vh\mid \vx)}{Q(\vh)} =: \LL(Q).
\end{align}
The above inequality holds for any distribution $Q(\vh)$,
but we get more usefulness out of it by choosing $Q(\vh)$ so that
it serves as a good approximation of $P(\vh\mid\vy,\vx)$.

We start by noting that we can use importance sampling to
express $P(\vh\mid\vy,\vx)$ in terms of a proposal distribution
${R(\vh |\vy,\vx)}$ from which we can draw samples.
\begin{align}
 P(\vh\mid\vy,\vx) \propto P(\vy\mid \vh) P(\vh \mid \vx)
  = \frac{P(\vy\mid \vh) P(\vh \mid \vx)}{R(\vh\mid \vy,\vx)} R(\vh\mid \vy,\vx)
\end{align}
Let $\delta(\vh)$ be the Dirac delta function centered at $\vh$. We construct $Q(\vh)$ based on this expansion:
\begin{align}
  Q(\vh) &= \sum_{m=1}^M \bar{w}^{(m)} \delta(\vh^{(m)}) \\
  \vh^{(m)} &\sim R(\vh\mid \vy,\vx) \\
  w^{(m)} &= \frac{P(\vy\mid \vh^{(m)}) P(\vh^{(m)} \mid \vx)}{R(\vh^{(m)}\mid \vy,\vx)} \\
  \bar{w}^{(m)} &= \frac{w^{(m)}}{\sum_{m^\prime=1}^M w^{(m^\prime)}}
\end{align}
where $w^{(m)}$ and $\bar{w}^{(m)}$ are called the unnormalized and normalized important weights.

It would be an interesting line of research to train an auxiliary model for the proposal distribution $R(\vh |\vx,\vy)$ following ideas from \cite{kingma2013auto,rezende2014stochastic,mnih2014neural} that call the equivalent of $R$ the recognition model or the inference network. However, we do not pursue that line further in this paper and follow
\citet{tang2013learning} who chose $R(\vh |\vx,\vy) := P(\vh |\vx)$, in which case the importance weights simplify to $w^{(m)} = P(\vy|\vh^{(m)})$.

\citet{tang2013learning} use a generalized EM algorithm, where they compute the gradient for the lower bound $\LL(Q)$ given that $Q(\vh)$ is fixed
\begin{align}
  G_4({\boldsymbol{\theta}}) &:= \frac{\partial}{\partial {\boldsymbol{\theta}}} \sum_\vh Q(\vh) \log \frac{P(\vy,\vh\mid \vx)}{Q(\vh)} \\
   &= \frac{\partial}{\partial {\boldsymbol{\theta}}} \sum_{m=1}^M \bar{w}^{(m)} \left[\log P(\vy\mid \vh^{(m)}) + \log P(\vh^{(m)}\mid \vx)\right].
\end{align}
Thus, we train $P(\vh^{(m)}|\vx)$ using $\vh^{(m)}$ as target outputs.

It turns out that the resulting gradient for $P(\vy|\vh)$ is exactly the same as in Section \ref{sec:P_y_given_h}, despite the rather different way of obtaining it.
The importance weights $\bar{w}^{(m)}$ have the same role as responsibilities $\bar{w}^{(m)}$ in the mixture model, so we can use the same notation for them.

{\bf Proposed Unbiased Estimator of the Gradient}
We propose a new gradient estimator by applying a variance reduction technique  \citep{weaver2001optimal,mnih2014neural} to the estimator by \citet{tang2013learning}.
First we note that 
\begin{align}
\EE_{P(\vh\mid \vx)} \left[\frac{\partial}{\partial {\boldsymbol{\theta}}} \log P(\vh\mid \vx) \right] 
 = \int P(\vh\mid \vx) \frac{\frac{\partial}{\partial {\boldsymbol{\theta}}} P(\vh\mid \vx)}{P(\vh\mid \vx)} \mathrm{d}\vh
 = \frac{\partial}{\partial {\boldsymbol{\theta}}} \int P(\vh\mid \vx) \mathrm{d}\vh = \frac{\partial}{\partial \theta} 1 = 0.
\end{align}
That is, when training $P(\vh| \vx)$ with samples $\vh^{(m)}\sim P(\vh| \vx)$ drawn from the model distribution, the gradient is on average zero. 
Therefore we can change the estimator of $P(\vh| \vx)$ by subtracting any constant $c$ from the weights $\bar{w}^{(m)}$ without introducing any bias.
We choose $c=\EE \left[\bar{w}^{(m)}\right]=\frac{1}{M}$
which is empirically shown to be sufficiently close to the optimum (see Figure \ref{fig:choosing_c_error_m} (left)).
Finally, the proposed estimator becomes
\begin{align}
  G_5({\boldsymbol{\theta}}) &:= \frac{\partial}{\partial {\boldsymbol{\theta}}} \sum_{m=1}^M \left[ \bar{w}^{(m)} \log P(\vy\mid \vh^{(m)})
  + \left(\bar{w}^{(m)} - \frac{1}{M}\right) \log P(\vh^{(m)}\mid \vx)\right].
  \label{eq:G5}
\end{align}

\begin{figure}[t]
\begin{center}
\includegraphics[width=0.49\textwidth, trim=150 330 150 335, clip]{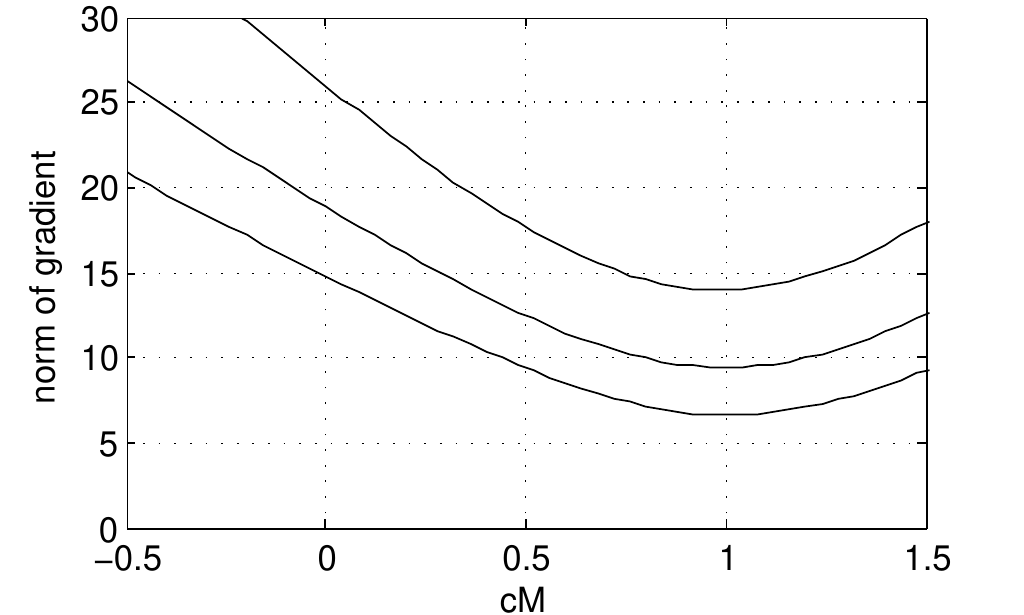} 
\includegraphics[width=0.49\columnwidth]{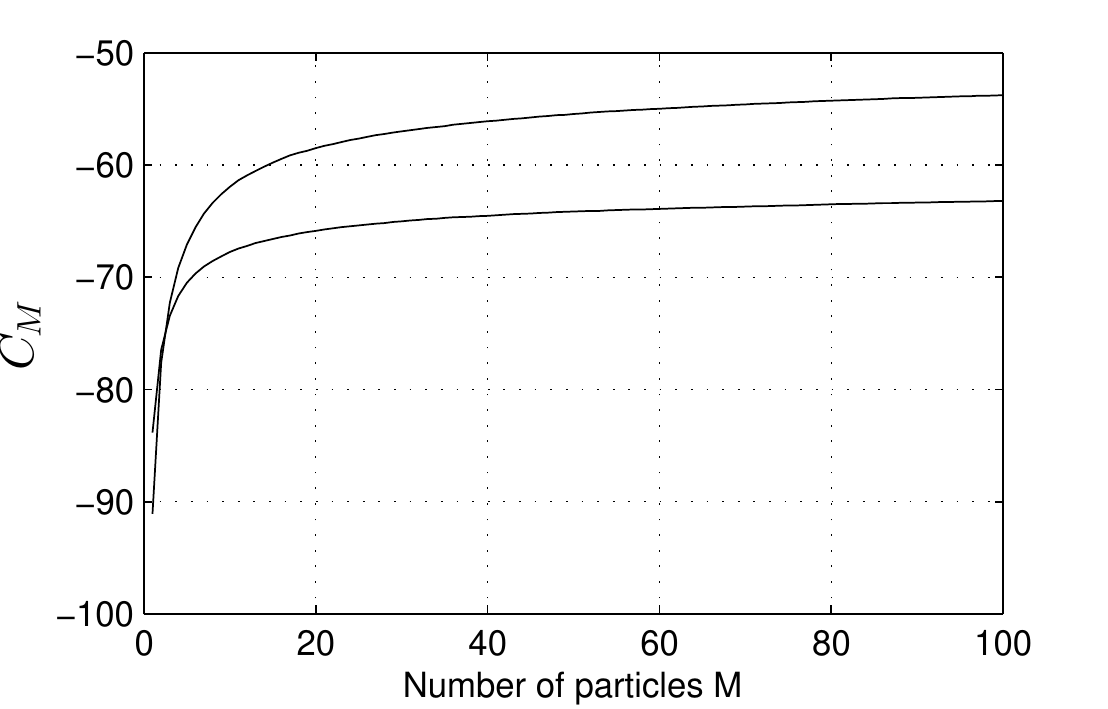} 

\caption{\textit{Left}: The norm of the gradient for the weights of the first hidden layer as a function of $cM$ where the proposed $c=\frac{1}{M}$ in Equation \eqref{eq:G5} corresponds to $cM=1$. The norm is averaged over a mini-batch, after \{1,7,50\} epochs of training (curves from top to bottom) with $G_5$ in the MNIST classification experiment (see Appendix~\ref{sec:classification}). Varying $c$ only changes the variance of the estimator, so the minimum norm corresponds to the minimum variance. \textit{Right}: $\hat{\CC}_{M}$ as a function of the number of particles used during test time for the MNIST structured prediction task for the two proposed models trained with $M=20$.}
\label{fig:choosing_c_error_m}
\end{center}
\end{figure}

\section{Experiments}

We propose two experiments as benchmarks for stochastic feedforward networks based on the MNIST handwritten digit dataset \citep{lecun1998gradient} and the Toronto Face Database \citep{tfd_report}. In both experiments, the output distribution is likely to be complex and multimodal.

In the first experiment, we predicted the lower half of the MNIST digits using the upper half as inputs. The MNIST dataset used in the experiments was binarized as a preprocessing step by sampling each pixel independently using the grey-scale value as its expectation. In the second experiment, we followed \citet{tang2013learning} and predicted different facial expressions in the Toronto Face Database \citep{tfd_report}. As data, we used all individuals with at least 10 different facial expression pictures, which we do not binarize. We set the input to the mean of these images per subject, and as output predicted the distribution of the different expressions of the same subject\footnote{We hence discarded the expression labels}. We randomly chose 100 subjects for the training data (1372 images) and the remaining 31 subjects for the test data (427 images). As the data in the second problem is continuous, we assumed unit variance Gaussian noise and thus trained the network using the sum of squares error. We used a network structure of 392-200-200-392 and 2304-200-200-2304 in the first and second problem, respectively.

Before running the experiments, we did a simple viability check of the gradient estimators by training a network to do MNIST classification. Based on the results, we kept $G_3$ to $G_5$ that performed significantly better than $G_1$ and $G_2$. The results of the viability experiment can be found in Appendix~\ref{sec:classification}.

For comparison, we also trained four additional networks (labeled A-D) in addition to the stochastic feedforward networks. Network~A is a deterministic network (corresponding to $G_3$ with $\epsilon_i=0$ in Equation \eqref{eq:epsilon}). In network~B, we used the weights trained to produce deterministic values for the hidden units, but instead of using these deterministic values at test time we use their stochastic equivalent. We therefore trained the network in the same way as network~A, but ran the tests as the network would be a stochastic network. Network~C is a hybrid network inspired by \citet{tang2013learning}, where each hidden layer consists of 40 binary stochastic neurons and 160 deterministic neurons. However, the stochastic neurons have incoming connections from the deterministic input from the previous layer, and outgoing connections to the deterministic neurons in the same layer. As in the original paper, the network was trained using the gradient estimator $G_4$. Network~D is the same as the hybrid network~C with one difference: the stochastic neurons have a constant activation probability of 0.5, and do hence not have any incoming weights or biases to learn. 

%

In all of the experiments, we used stochastic gradient descent with a mini-batch size of 100 and momentum of 0.9. We used a learning rate schedule where the learning rate increases linearly from zero to maximum during the first five epochs and back to zero during the remaining epochs. The maximum learning rate was chosen among $\{0.0001, 0.0003, 0.001, \dots, 1\}$ and the best test error for each method is reported.\footnote{In the MNIST experiments we used a separate validation set to select the learning rate. However, as we chose just one hyperparameter with a fairly sparse grid, we only report the best test error in the TFD experiments without a separate validation set.} The models were trained with $M\in\{1,20\}$, and during test time we always used $M=100$.


As can be seen in Table~\ref{tab:structured},
excluding the comparison methods, the proposed biased estimator $G_3$ performs the best in both tasks.
It is notable that the performance of $G_3$ increased
significantly when using more than $M=1$ particles,
as could be predicted from Theorem 1.
In Figure~\ref{fig:choosing_c_error_m} (right) we plot the objective $\CC_M$ at
test time based on a number of particles $M=1,\ldots,100$.
In theory, a larger number of particles $M$ is always better
(if given infinite computational resources), but here
Figure~\ref{fig:choosing_c_error_m} (right) shows how the objective $\CC_M$
is estimated very accurately with only $M=20$ or $M=40$.

Of all the networks tested, the best performing network in both tasks was however comparison network~D, i.e. the deterministic network with added binary stochastic neurons that have a constant activation probability of 0.5. It is especially interesting to note that this network also outperformed the hybrid network~C where the output probabilities of the stochastic neurons are learned. Network~D seems to gain from being able to model stochasticity without the need to propagate errors through the binary stochastic variables. The results give some support to the hypothesis that a hybrid network outperforms a stochastic network because it is easier to learn a deterministic network around a small number of stochastic units than learning a full stochastic network, although the stochastic units are not trained properly.

The results could possibly be improved by making the networks larger and continuing training longer if given enough computational capacity. This might be the case especially in the experiments with the Toronto Face Dataset, where the deterministic network~A outperforms some of the stochastic networks. However, the critical difference between the stochastic networks and the deterministic network can be be observed in Figure~\ref{fig:reconstructions}, where the stochastic networks are able to generate reconstructions that correspond to different digits for an ambiguous input. Clearly, the deterministic network cannot model such a distribution.

\begin{figure*}[t]
\begin{center}
\includegraphics[width=0.4218\textwidth]{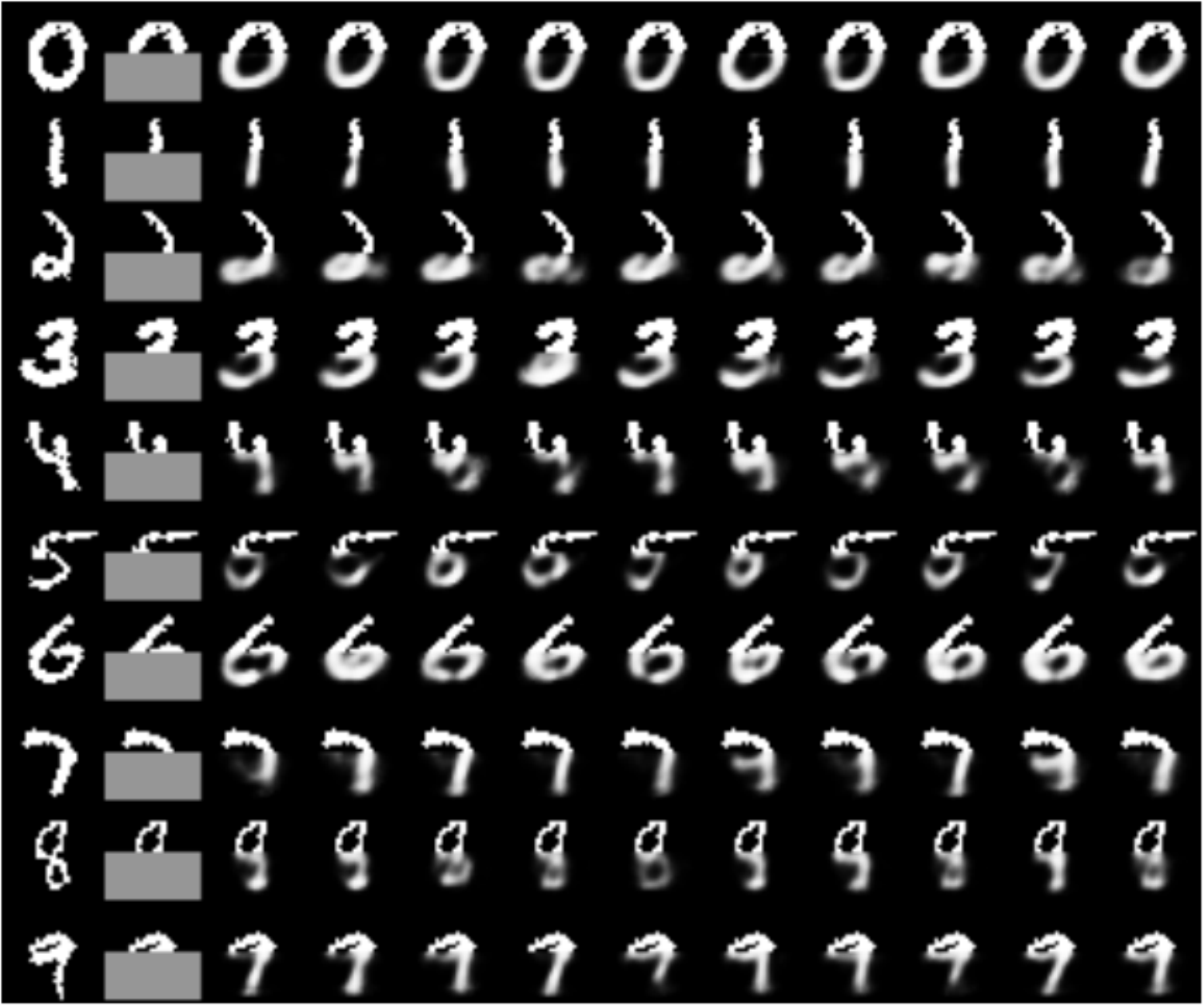} 
\includegraphics[width=0.4218\textwidth]{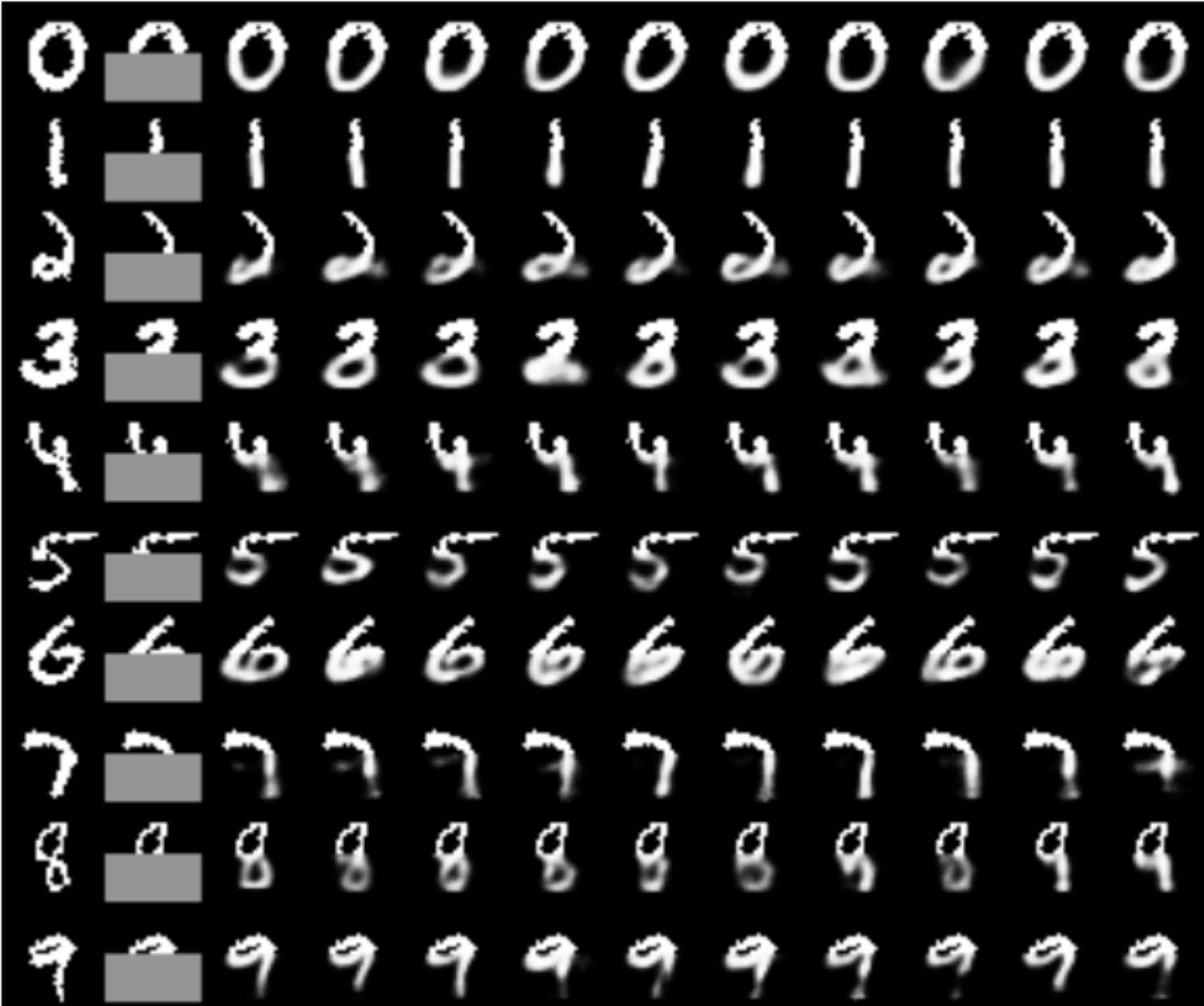}
\includegraphics[width=0.1070\textwidth]{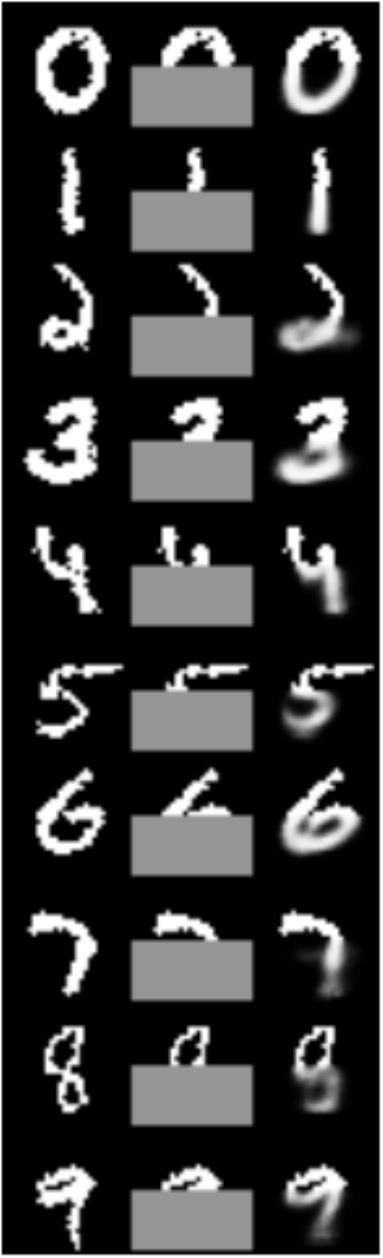}  

\caption{Samples drawn from the prediction of the lower half of the MNIST test data digits based on the upper half with models trained using $G_3$ (left), $G_5$ (middle), and for the deterministic network (right). The leftmost column is the original MNIST digit, followed by the masked out image and ten samples. The figures illustrate how the stochastic networks are able to model different digits in the case of ambiguous inputs.}
\label{fig:reconstructions}
\end{center}
\end{figure*}

\newcommand{\specialcell}[2][c]{\begin{tabular}[#1]{@{}c@{}}#2\end{tabular}}

\begin{table}[t]
\caption{
Results obtained on MNIST and TFD structured prediction using various number of samples $M$ during training and various estimators of the gradient $G_i$. Error margins are $\pm$ two standard deviations from 10 runs.
}
\begin{center}
\begin{tabular}{l | c c | c c }

& \multicolumn{2}{c|}{\specialcell{MNIST Neg. test \\ log-likelihood ($\hat{\CC}_{100}$)}} & \multicolumn{2}{c}{\specialcell{TFD test Sum of \\ Squared Errors}} \\

& $M=1$ & $M=20$ & $M=1$ & $M=20$ \\ \hline

$G_3$ (proposed biased) & $59.8 \pm 0.1$ & $53.8 \pm 0.2$ & $31.7 \pm 0.7$ & $26.3 \pm 3.7$ \\  
$G_4$ (Tang \textit{et al.}, \citeyear{tang2013learning}) & na & $64.0 \pm 1.7$ & na & $51.4 \pm 0.1$ \\
$G_5$ (proposed unbiased) & na & $63.2 \pm 1.2$ & na & $51.3 \pm 0.1$ \\ \hline
deterministic (A) & $68.4 \pm 0.1$ & na & $35.3 \pm 0.4$ & na \\
deterministic as stochastic (B) & $59.1 \pm 0.2$ & na & $48.3 \pm 7.5$ & na \\ 
hybrid (C) & na & $58.4 \pm 0.8$ & na & $35.5 \pm 1.0$ \\
deterministic, binary noise (D) & $67.9 \pm 1.1$ & $\mathbf{52.0 \pm 0.2}$ & $33.4 \pm 0.6$ & $\mathbf{21.4 \pm 0.6}$ \\

\end{tabular}
\label{tab:structured}
\end{center}
\end{table}

\section{Discussion}


In the proposed estimator of the gradient for $P(\vh|\vx)$ in Equation \eqref{eq:G5}, 
there are both positive and negative weights for various particles $\vh^{(m)}$.
Positive weights can be interpreted as pulling probability mass towards the particle, and negative weights as pushing probability mass away from the particle.
Although we showed that the variance of the gradient estimate is smaller when using both positive and negative weights ($G_5$ vs. $G_4$), the difference in the final performance of the two estimators was not substantial

One challenge with structured outputs $\vy$ is to find samples $\vh^{(m)}$ that give a reasonably large probability $P(\vy| \vh^{(m)})$ with a reasonably small sample size $M$. 
Training a separate $R(\vh | \vx,\vy) \neq P(\vh | \vx)$ as a proposal distribution looks like a promising direction for addressing that issue. 
It might still be useful to use a mix of particles from $R$ and $P(\vh | \vx)$, and subtract a constant from the weights of the latter ones. 
This approach would yield both particles that explain $\vy$ well, and particles that have negative weights.


\section{Conclusion}

Using stochastic neurons in a feedforward network is more than just a
computational trick to train deterministic models.
The model itself can be defined in terms of stochastic particles in the hidden layers,
and we have shown many valid alternatives 
to the usual gradient formulation.

These proposals for the gradient involve particles in the hidden layers with normalized weights that represent how
well the particles explain the output targets. We showed both theoretically and experimentally how involving more than one particle significantly enhances the modeling capacity.

We demonstrated the validity of these techniques in three sets of experiments: we trained a classifier on MNIST that achieved a reasonable performance,
a network that could fill in the missing information
when we deleted the bottom part of the MNIST digits, and a network that could output individual expressions of face images based on the mean expression.

We hope that we have provided some insight into the properties of
stochastic feedforward neural networks, and that the theory
can be applied to other contexts such as the study of Dropout
or other important techniques that give a stochastic flavor
to deterministic models.

\subsubsection*{Acknowledgements}

The authors would like to acknowledge NSERC, Nokia Labs and the Academy of Finland as sources of funding, in addition to the developers of Theano \citep{Bastien-Theano-2012, bergstra+al:2010-scipy}

\bibliography{read}
\bibliographystyle{iclr2015}

\newpage
\appendix

\section{Classification experiment}
\label{sec:classification}

MNIST classification is a well studied problem where performances of a huge variety of approaches are known.
Since the output $\vy$ is just a class label, the advantage of being able to model complex output distributions is not applicable. Still, the benchmark is useful for comparing training algorithms against each other, and was used in this paper to test the viability of the gradient estimators.

We used a network structure with dimensionalities 784-200-200-10. The input data was first scaled to the range of $[0,1]$, and the mean of each pixel was then subtracted. 
As a regularization method, Gaussian noise with standard deviation 0.4 was added to each pixel separately in each epoch \citep{raiko2012deep}. The models were trained for 50 epochs.

Table~\ref{tab:mnist_classification} gives the test set error rate for each method. As can be seen from the table, deterministic networks give the best results. Excluding the comparison networks, the best result is obtained with the proposed biased gradient $G_3$ followed by the proposed unbiased gradient $G_5$. Based on the results, gradient estimators $G_1$ and $G_2$ were left out from the structured prediction experiments.

\begin{table}[h]
\caption{
Results obtained on MNIST classification using various number of samples $M$ during training and various estimators of the gradient $G_i$.
}
\begin{center}
\begin{tabular}{|l | c | c |}
\hline
\centering Test error (\%) & $M=1$ & $M=20$ \\
\hline
$G_1$ \citep[unbiased]{bengio2013estimating} & 7.85 & 11.30 \\ 
$G_2$ \citep[biased]{bengio2013estimating} & 7.97 & 7.86 \\ 
$G_3$ (proposed biased) & 1.82 & 1.63 \\ 
$G_4$ (Tang \textit{et al.}, \citeyear{tang2013learning}) & na & 3.99 \\ 
$G_5$ (proposed unbiased) & na & 2.72 \\ \hline
deterministic (A) & \bf{1.51} & na \\
deterministic as stochastic (B) & 1.80 & na \\
hybrid (C) & na & 2.19 \\ 
deterministic, binary noise (D) & 1.80 & 1.92 \\ \hline
\end{tabular}

\label{tab:mnist_classification}
\end{center}
\end{table}

\end{document}